\documentclass{IOS-Book-Article}
\usepackage[utf8]{inputenc}
\usepackage{times}
\usepackage[square,sort,comma,numbers]{natbib}
\usepackage{todonotes}
\usepackage{subcaption}
\usepackage{mathtools}
\usepackage{amsfonts,amsmath,amsthm,bbm}
\usepackage{sgame}
\usepackage{graphicx}
\usepackage{algorithm,algpseudocode}

\usepackage{url}
\usepackage{tikz,pgfplots,tikz-3dplot}
\usetikzlibrary{calc}

 \newtheorem{definition}{Definition}
 \newtheorem{example}{Example}
 \newtheorem{proposition}{Proposition}
 \newtheorem{lemma}{Lemma}
 \newtheorem{corollary}{Corollary}

\newcommand{\AS}{\mathcal{AS}}
\newcommand{\A}{\mathcal{A}}
\newcommand{\Arg}{\mathbb{A}}
\newcommand{\Att}{\mathbb{C}}
\newcommand{\C}{\mathcal{C}}
\newcommand{\Datt}{\mathbb{D}}

\newcommand{\Ext}{Ext}
\newcommand{\ag}{Ag}
\newcommand{\ac}{Ac}
\newcommand{\av}{Av}
\newcommand{\ou}{Ou}
\newcommand{\ef}{Ef}
\newcommand{\pref}{\leq}
\newcommand{\Strat}{S}
\renewcommand{\gamestretch}{2.1}
\newcommand{\stackedpayoffs}[2]{\begin{array}{cc}  & #1 \\[1.6mm]  #2\phantom{-}&\end{array} }

\begin{document}
\pagestyle{headings}
\begin{frontmatter}

\title{Representing Pure Nash Equilibria in Argumentation}
\author[A]{\fnms{Bruno} \snm{Yun}
\thanks{Corresponding Author: E-mail: bruno.yun@abdn.ac.uk}},
\author[B]{\fnms{Srdjan} \snm{Vesic}}
and
\author[A]{\fnms{Nir} \snm{Oren}}

\runningauthor{B.~Yun et al.}
\address[A]{Department of Computing Science, University of Aberdeen, Scotland}
\address[B]{CRIL Lens, University of Artois, France}




\begin{abstract}
In this paper we describe an argumentation-based representation of normal form games, and demonstrate how argumentation can be used to compute pure strategy Nash equilibria. Our approach builds on Modgil's Extended Argumentation Frameworks. We demonstrate its correctness, prove several theoretical properties it satisfies, and outline how it can be used to explain why certain strategies are Nash equilibria to a non-expert human user.
\end{abstract}

\end{frontmatter}

\section{Introduction}
Game theory studies how multiple rational decision-makers should act given interactions between their strategies, and preferences over the resultant outcomes. Game theory has been applied to myriad fields \cite{matsumoto2016GameTheoryIts}. Within game theory, decision-makers (referred to as players), their strategies, preferences and outcomes are represented within a game, and the solutions to a game identify some form of rational outcome. One such solution concept is that of a \emph{dominant} strategy, where a player has a strategy or a set of strategies that will always result in the best outcome for them, regardless of what other players do. However, such dominant strategies often do not exist.
In this work, we consider instead the notion of a \emph{Nash equilibrium}, which identifies optimal strategies given that other players also pursue their own optimal strategies. Such Nash equilibria therefore represent a form of best response, and provide a well understood solution concept in game theory. However, finding Nash equilibria is computationally difficult, and it is sometimes difficult for a non-expert to understand why a given strategy is (or is not) a Nash equilibrium. We believe that by providing an argumentation-based representation of games, dialogues can be used to explain a Nash equilibrium to such non-experts. While work such as \cite{fanInterplayGamesArgumentation2016} has considered game theory in the context of ABA, to our knowledge, this work is the first to link abstract argumentation and Nash equilibria. We consider only so-called \emph{pure strategies} for \emph{normal form games} and intend to relax this restriction in future work.

The remainder of the paper is structured as follows. In Section \ref{sec:background}, we provide a brief overview of argumentation and game-theory concepts necessary to understand our article. In Section \ref{sec:formalism}, we describe how a normal form game can be encoded using argumentation. Section \ref{sec:properties} examines some formal properties of our approach. Lastly, we discuss related and future work in Section \ref{sec:discussion} before concluding.

\section{Background}\label{sec:background}

We begin by providing the necessary background in game theory and argumentation required for the rest of the paper.

\subsection{Game Theory}

In this paper, we use the usual \emph{normal form} for games \cite{2009MartinOsborneIntroduction}. 

\begin{definition}\textbf{(Normal Game)}
  A (normal) game is $G = (\ag,\ac, \av, \ou, \ef, \pref)$ where $\ag = \{0, 1, \dots, n\}$ is a finite set of players; $\ac$ is a finite set of strategies; $\av = [\ac_0, \dots, \ac_n ]$ with $\ac_i \subseteq \ac$ denoting the strategies available to $i$; $\ou = \{o_0, \dots, o_m\}$ is a set of possible outcomes; $\ef:\ac^{n} \to \ou^{n}$ captures the consequences of the joint strategies for each player; and $\pref = [ \pref_0, \dots, \pref_n]$ with $\pref_i \subseteq \ou \times \ou$ denoting the preference relation for player $i$.
\end{definition}

The notation $o_k \pref_i o_l$ means that player $i$ prefers outcome $o_l$ to $o_k$. As commonly done, we write $o_i <_i o_j$ iff $o_i \leq_i o_j$ and $o_j \not\leq_i o_i$\footnote{We assume that our preferences are acyclic. I.e., if $a <_i b <_i c$ then $c \not\pref_i a$.}.
A \emph{pure strategy profile} $\Strat$ is a tuple containing one strategy from each player in the game. The set of all such pure strategy profiles is $\Strat_G=\Pi_{i \in \ag} \ac_i$, and represents one joint strategy of all players. 
A \emph{partial strategy profile} is a tuple containing a single strategy for a subset of the players.
Given any pure strategy profile $\Strat=[s_0, \ldots, s_n]$, we write $\Strat_{-i}$ to denote the \emph{partial strategy profile} $[s_0, \ldots, s_{i-1}, \emptyset, s_{i+1}, \ldots, s_n]$, where the strategy for player $i$ is not specified. We then write $\Strat_{-i} \oplus s_i$ to denote strategy profile $\Strat$.
With a slight abuse of notation, for any $\Strat, \Strat' \in \Strat_G$ we write that $S   \leq_i S'$ iff $\ef(S)_i \leq_i \ef(S')_i$\footnote{The notation $\ef(S')_i$ means the i-th element of $\ef(S')$.}.

\begin{example}
Let us consider the stag hunt game $G = (\{0,1\}, \ac, \av, \ou, \ef, \pref)$, where $\ac = \{ stag, hare\}$, $\av = [\ac, \ac]$, $\ou =\{ 4, 3, 2, 1 \}$, $\pref$ is the standard less than relation over numbers. Table \ref{tab:stag} graphically illustrates this game in normal form, and specifies $\ef$. For example, the tuple $(1,3)$ in the column ``hare'' and row ``stag'' means that $\ef([stag,hare]) = (1,3)$.
Given the pure strategy profile $\Strat = [stag,hare]$, $\Strat_{-0} = [\emptyset, hare]$ and $\Strat_{-0} \oplus hare = [hare,hare]$.
Here $[stag,hare] \leq_0 [hare,hare]$ because $(1,3)_0 \leq_0 (2,2)_0$ but $ [hare,hare] \leq_1 [stag,hare]$.

\begin{table}[t]
\centering
\begin{subtable}{0.4\textwidth}
\begin{game}{2}{2}[Player~0][Player~1]
& $stag$ & $hare$ \\
$stag$ &$\stackedpayoffs{4}{4}$ &$\stackedpayoffs{3}{1}$\\
$hare$ &$\stackedpayoffs{1}{3}$ &$\stackedpayoffs{2}{2}$
\end{game}
\caption{Stag Hunt} \label{tab:stag}
\end{subtable}
\begin{subtable}{0.4\textwidth}
 \begin{game}{2}{2}[Player~0][Player~1]
 & $heads$ & $tails$ \\
 $heads$ &$\stackedpayoffs{-1}{1}$ & $\stackedpayoffs{1}{-1}$\\
$tails$ &$\stackedpayoffs{1}{-1}$  &$\stackedpayoffs{-1}{1}$
 \end{game}
 \caption{Matching pennies.}
 \label{tab:noNash}
 \end{subtable}
\caption{Two games in normal form.}
\end{table}

\label{ex:stag1}
\end{example}

In asking why a player should pursue a some strategy, we must take into account the strategies of others. 
A \textit{Nash equilibrium} is the best response a player can make given optimal play by all other players.

\begin{definition}
Let $G= (\ag,\ac,\av, \ou, \ef, \pref)$, we say that $\Strat \in \Strat_G$ is a Nash equilibrium if for every $ i \in \ag$ and for any strategy $s \in \ac_i$, it holds that $\Strat_{-i} \oplus s \pref_i  \Strat$.

\end{definition}

A simple algorithm to identify all Nash equilibrium in the presence of pure strategies involves iterating through every player and identifying the best strategy profile (in terms of $\ef$ for that player) given all other players' possible joint strategies. Any strategy profile which all players consider best is then a Nash equilibrium.

Given a game in normal form, the above algorithm involves -- for a two player game --  scanning down each column and marking the best strategy for the row player, and then doing the same for each row marking the best strategy for the column player. Each cell marked for both players is a Nash equilibrium.
In the remainder of this paper, we show an argumentation-based alternative.

\begin{example}[Cont'd]
There are two Nash equilibria in the stag hunt game: $[stag,stag]$ and $[hare,hare]$. The strategy profile $[stag,stag]$ is a Nash equilibrium because $[ hare, stag] \leq_0 [stag, stag]$ and $[stag, hare] \leq_1 [stag,stag] $.
Similarly, $[hare,hare]$ is also a Nash equilibrium as $[stag,hare] \leq_0 [hare,hare] $ and $[hare,stag] \leq_1 [hare,hare]$.
\end{example}

\subsection{Argumentation}

We encode normal form games in terms of arguments and attacks by building on Modgil's Extended Argumentation Frameworks (EAF) \cite{modgil2009ReasoningPreferencesArgumentation}.

\begin{definition}
An Extended Argumentation Framework is a triple $\langle \Arg , \Att, \Datt \rangle$ where $\Arg$ is a set of arguments, $\Att \subseteq \Arg \times \Arg$, $\Datt \subseteq \Arg \times \Att$ and if $(z,(x,y)), (z',(y,x)) \in \Datt$ then $(z,z'), (z',z)  \in \Att$.
\end{definition}

\begin{definition}[Defeat]

Let $\AS = (\Arg, \Att, \Datt)$ be an EAF, $x,y \in \Arg$ and $Y \subseteq \Arg$. We say that $y$ defeats $x$ w.r.t.\ $Y$, denoted $y \to_Y x$ iff $(y,x) \in \Att$ and there is no $z \in Y$ s.t.\ $(z,(y,x)) \in \Datt$.
\end{definition}

\begin{definition}[Argumentation semantics]
Let $\AS = (\Arg,\Att, \Datt)$ be an EAF and $E \subseteq \Arg$. We say that:
\begin{itemize}
    \item $E$ is conflict-free iff for every $x,y \in E$, if $(y,x) \in \Att$ then $(x,y) \not\in \Att$, and there exists $z \in E$ s.t.\ $(z,(y,x)) \in \Datt$.
    \item $x \in \Arg$ is acceptable w.r.t.\ $E$ iff for every $y \in \Arg$ s.t.\ $y \to_E x$, there exists $z \in E$ s.t.\ $z \to_E y$ and there exists $R_E = \{ x_1 \to_E y_1, \dots , x_n \to_E y_n\}$ s.t.\ for every $i\in \{1, \dots, n \}, x_i \in E$, $z \to_E y \in R_E$ and for every $x_j \to_E y_j \in R_E$, for every $y'$ s.t.\ $(y',(x_j,y_j)) \in \Datt$, there exists $x' \to_E y'  \in R_E$
    \item $E$ is an admissible extension iff every argument in $E$ is acceptable w.r.t.\ $E$
    \item $E$ is a preferred extension iff $E$ is a maximal (w.r.t.\ $\subseteq$) admissible extension
    \item $E$ is a stable extension iff for every $y \notin E$, there exists $x \in E$ such that $x \to_E y$.
\end{itemize}
\end{definition}

We will use the notation $\Ext_s(\AS)$ (resp. $\Ext_p(\AS)$) to denote the set of all stable (resp.\ preferred) extensions.

\section{Argumentation-based approach for games} \label{sec:formalism}

We consider an argumentation framework with multi-level arguments. At the base level, we consider all possible strategy profiles as arguments. Since only a single strategy profile can ever occur (as players execute one set of strategies in the interaction), every argument at this level must attack every other argument. We refer to such arguments as \emph{game-based arguments}, and note that they are equivalent to pure strategy profiles.

\begin{definition}[Game-based argument]
Let $G = (\ag,\ac, \av,\ou, \ef, \pref)$ be a game, a game-based argument (w.r.t.\ $G$) is a pure strategy profile $\Strat \in \Strat_G$.
\label{def:game-args}
\end{definition}

\noindent The set of all game-based arguments for a game $G$ is denoted by $\A_g(G)$.

Next, we introduce \emph{preference arguments}. Intuitively, these can be interpreted as statements of the form: ``Given that the other players are performing a given set of strategies, the remaining player's preferred strategy should be playing $x$''.

\begin{definition}[Preference argument]
Let $G = (\ag,\ac,\av, \ou, \ef, \pref)$ be a game, $\Strat \in \Strat_G$ be a pure strategy profile and $i \in \ag$. A preference argument  (w.r.t.\ $G$) is a tuple $(\Strat_{-i},s)$, where $s \in \ac_i$.
\label{def:pref-args}
\end{definition}

\noindent The set of preference arguments for a game $G$ is denoted by $\A_p(G)$. 
A \emph{cluster} of preference arguments is a maximal set of preference arguments sharing the same partial strategy profile.
%
%

Finally, we introduce \emph{valuation arguments}, which can be interpreted as statements of the form: ``Given that the other players are performing a given set of strategies, it is the case that the outcome of strategy $s$ is better than the outcome of strategy $s'$ for the remaining player''.

\begin{definition}[Valuation argument]
Let $G = (\ag,\ac,\av, \ou, \ef, \pref)$ be a game, $i \in \ag$, $(\Strat_{-i},s), (\Strat_{-i},s') \in \A_p(G)$ be two preference arguments and $\Strat_{-i} \oplus s'  <_i \Strat_{-i} \oplus s$. A valuation argument (w.r.t.\ $G$) is the pair $(\Strat_{-i}, s' < s)$. 
\end{definition}

\noindent The set of valuation arguments for a game $G$ is denoted by $\A_v(G)$.

\begin{example}[Cont'd]
\label{ex3-stag-hare}
The sets of game-based, preference and valuation arguments w.r.t. $G$ are shown in Table \ref{tab:args-stag}.
The argument $a_1$ represents the case where player $0$ chooses to hunt a stag and player $1$ chooses to hunt a hare. The argument $a_9$ represents the argument: ``Given that player $0$ chooses to hunt a hare, player $2$'s preferred strategy should be to hunt a stag''.
The argument $a_{16}$ represents the argument: ``Given that player $1$ chooses to hunt a hare, the outcome of hunting a hare is better than the outcome of hunting a stag for player $0$''.

\begin{table}
    \centering
\begin{tabular}{|c|c|c|}
\hline
Game-based arguments & Preference arguments & Valuation arguments\\
\hline
$a_1 = [stag,hare]$ &$a_5 = ([stag,\emptyset], stag)$ & $a_{13} = ([stag,\emptyset], stag > hare)$ \\
$a_2 = [stag,stag]$ & $a_6 = ([stag,\emptyset], hare)$& $a_{14} = ([\emptyset,stag], stag > hare)$ \\
$a_3 = [hare, stag]$ & $a_7 = ([\emptyset,stag], stag)$& $a_{15} = ([hare,\emptyset], hare > stag)$ \\
$a_4 = [hare, hare]$ & $a_8 = ([\emptyset,stag], hare)$& $a_{16} = ([\emptyset, hare], hare > stag)$. \\
& $a_9 = ([hare,\emptyset], stag)$& \\
&$a_{10} = ([hare,\emptyset], hare)$& \\
&$a_{11} = ([\emptyset,hare], stag)$&\\
&$a_{12} = ([\emptyset,hare], hare)$ & \\
\hline
\end{tabular}
    
    \caption{Arguments for the stag hunt game}
    \label{tab:args-stag}
\end{table}
\end{example}

We now turn our attention to attacks. We note that preference and valuation arguments provide reasons why one argument should not attack another, and therefore introduce not only attacks between arguments, but also attacks on attacks.

\begin{definition}[Attack]
\label{defintion-of-attack}
For a game $G = (\ag,\ac,\av, \ou, \ef, \pref)$, $a_1, a_2 \in \A_g(G)$, $a_3 = (\Strat_1,s_2) , a_4= (\Strat_3,s_4) \in \A_p(G)$ and $a_5 = (\Strat_5, s_6 > s_7) \in \A_v(G)$. We say that:

\begin{itemize}
    \item $a_1$ attacks $a_2$, denoted $(a_1, a_2) \in \C_r(G)$, iff $a_1 \neq a_2$.
    
    \item $a_3$ attacks $a_4 $, denoted $(a_3, a_4) \in \C_p(G)$, iff $\Strat_1 = \Strat_3$ and $s_2 \neq s_4$.
    
    \item $a_3 $ attacks $(a_1, a_2) \in \C_r(G)$, denoted by $(a_3, (a_1,a_2)) \in \C_u(G)$, iff there exists $s \in \ac$ such that $\Strat_1 \oplus s =  a_1$ and $\Strat_1 \oplus s_2 = a_2$. 
    
    \item $a_5$ attacks $(a_3, a_4) \in \C_p(G)$, denoted by $(a_5, (a_3,a_4)) \in \C_v(G)$, iff $\Strat_5 = \Strat_3$, $s_6 = s_4$ and $s_7 = s_2$.
\end{itemize}
\label{def:attacks}
\end{definition}

The first attack captured within Definition \ref{def:attacks} is between every two distinct game-based arguments. As each player has to choose exactly one strategy, different strategy profiles are clearly incompatible. 
The second bullet point represents attacks between preference arguments. In the stag hunt example for instance, $a_5$ attacks $a_6$ (and vice-versa) because in the event of player $0$ hunting a stag, player $1$ can either hunt a stag or a hare.
The third type of attack captures attacks from preference arguments to attacks between game-based arguments. Within the stag hunt, $a_5$ attacks $(a_1, a_2)$ because $a_5$ states that it is preferable for player $1$ to hunt a stag when player $0$ is also hunting a stag.
Note that in general, the preference argument $(\Strat_1,s_2)$ attacks \emph{all} attacks against the game-based argument $\Strat_1 \oplus s_2$ coming from any other game-based arguments of the form $\Strat_1 \oplus s'$, for any $s' \in \ac$ such that $s' \neq s_2$.
The last type of attack captures  attacks from valuation arguments to attacks between preference arguments.
Returning to the stag hunt, $a_{13}$ attacks $(a_6, a_5)$ as $a_{13}$ states that the strategy ``hunt a stag'' is better than the strategy ``hunt a hare'' for player $1$ when player $0$ is hunting a stag.

The arguments and attacks induce a very specific type of extended argumentation framework, where object-level (game-based) arguments have their attacks attacked by meta-arguments (preference arguments) at level one, and where attacks between these meta-arguments are attacked by meta-arguments at level two (valuation arguments).

\begin{definition}[Argumentation framework]
Let $G$ be a game. The argumentation framework corresponding to $G$ is the tuple $\AS_G = (\Arg, \Att, \Datt)$ where $\Arg = \A_g(G) \cup \A_p(G) \cup \A_v(G)$, $\Att = \C_r(G) \cup \C_p(G)$ and $\Datt= \C_u(G) \cup \C_v(G)$.

\label{def:instantiation}
\end{definition}

\begin{example}[Example \ref{ex3-stag-hare} Contd]
Figure \ref{fig:GPd} represents the game-based, preference and valuation arguments of $G$ using blue, yellow and green nodes respectively. The attacks between arguments ($\Att$) and on attacks ($\Datt$) are represented using solid black arrows and dashed red arrows respectively.

\begin{figure}
    \centering
    \includegraphics[width=0.5\textwidth]{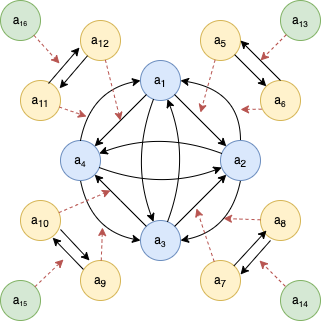}
    \caption{Argumentation graph corresponding to stag hunt game}
    \label{fig:GPd}
\end{figure}
\label{ex-fig-AS}
\end{example}

For our framework to be an EAF, it must satisfy some constraints, as described in \cite{modgilLabellingsGamesExtended}, and we can easily show that this is the case.


\begin{proposition}
Let $G$ be a game and $\AS_G = (\Arg, \Att, \Datt)$ be the corresponding argumentation framework, it holds that if $(z,(x,y)), (z',(y,x)) \in \Datt$ then $(z,z'), (z',z)  \in \Att$.
\label{prop:EAFcondition}
\end{proposition}

\begin{proof}
There are only two types of attacks on attacks: (1) attacks coming from valuation arguments to attacks between preference arguments and (2) attacks coming from preference arguments to attacks between game-based arguments.
In the rest of this proof, we prove that Proposition \ref{prop:EAFcondition} is satisfied for the two types of attacks on attacks.
\begin{itemize}
    \item 
Considering $(1)$, for a fixed partial strategy profile $\Strat_i$, and fixed strategies $s_j, s_k \in \ac$, there is exactly one (or no) valuation argument of the form $(\Strat_i, s_j > j_k)$ or $(\Strat_i, s_k > s_j)$. As a result, the condition in Proposition \ref{prop:EAFcondition} is trivially satisfied for attacks coming from valuation arguments.

\item We now study the case $(2)$ and show that Proposition \ref{prop:EAFcondition} is also satisfied for attacks coming from preference arguments on attacks between game-based arguments. 
Assume that $(a_3,(x,y)),$ $(a_4,(y,x)) \in \Datt$, where $a_3 = (\Strat_1,s_2),$ $a_4 = (\Strat_1, s_4), x = \Strat_1 \oplus s_4$ and $y = \Strat_1 \oplus s_2$. 
By Definition \ref{defintion-of-attack}, $s_2 \neq s_4$ thus $(a_3,a_4), (a_4, a_3) \in \C_p(G) \subseteq \Att$.
\end{itemize}

\end{proof}

\noindent Since -- given Proposition \ref{prop:EAFcondition} -- our argumentation system is an EAF, we can use EAF semantics to evaluate it.

\begin{example}[Example \ref{ex-fig-AS} Contd]
In our running example, $a_5$ defeats $a_6$ w.r.t.\ $\Arg$ as $(a_5,a_6) \in \Att$ and there is no argument $z \in \Arg$ such that $(z,(a_5,a_6)) \in \Datt$. However, $a_6$ does not defeat $a_5$ w.r.t.\ $\Arg$ because $(a_{13},(a_6,a_5)) \in \Datt$.
All extensions contain arguments $\{a_{16},a_{15},a_{14},a_{13},a_{12},a_{10},{a_7},a_{5}\}$, while one preferred extension contains $\{a_2\}$ and the other contains $\{a_4\}$.
\end{example}

\section{System Properties}\label{sec:properties}

Having described our system, we now consider its properties. The most important result we seek to show is the correspondence between argumentation semantics and Nash equilibria, and we begin by laying the groundwork for this. We then consider how many arguments will be generated for an arbitrary normal form game.

We begin by considering which preference arguments will appear in a preferred extension. This result is used in later proofs.

\begin{lemma}
Let $G = (\ag,\ac,\av, \ou, \ef, \pref)$ be a game, and $\AS_G$ be the corresponding AS. For each preferred extension $E$ of $\AS_G$, for each cluster $C$ of preference arguments, there exists a unique argument $c \in C$ such that $c \in E$.
\label{lem:1}
\end{lemma}

\begin{proof}
Assume a partial strategy profile $\Strat = [s_0, \dots, s_{i-1}, \emptyset, s_{i+1}, s_n]$ and the corresponding cluster of preference arguments $C$.
Because our preferences are acyclic, we know that there exists a strategy $s^*$ such that for every $s \in \ac_i$, $\Strat \oplus s \leq_i \Strat \oplus s^*$. From the definition of the valuation argument, there are no valuation arguments attacking the attacks from the preference argument $(\Strat, s^*)$ to other preference arguments. As a result, we conclude that $(\Strat, s^*)$ is in a preferred extension $E$ and that all the other preference in $C$ are not $E$.
Moreover, you need to choose one of such arguments from the cluster $C$ for each preferred extension to satisfy the maximality condition of the semantics.
\end{proof}

Next, we show that if there is a preferred extension with game-based arguments, then each such extension has exactly one game-based argument.

\begin{lemma}
\label{lem:pref-gameb}
If any preferred extension of $\AS_G$ contains a game-based argument, then it contains exactly one game-based argument.
\end{lemma}

\begin{proof} 
Let $E$ be a preferred extension containing game-based arguments. We prove by contradiction that it is not possible for $E$ to have more than one game-based argument.
Assume that $E$ contains two game-based arguments $a_1$ and $a_2$. By definition of the attack relation, there is a symmetric attack between $a_1$ and $a_2$. Hence there must exist two preference arguments $p_3$ and $p_4$ with $(p_3, (a_1,a_2)), (p_4 (a_2,a_1)) \in \Datt$ and $(p_3,p_4), (p_4,p_3) \in \Att$. 
It is not possible for both $(p_4,p_3)$ and $(p_3,p_4)$ to be attacked by valuation arguments as this would require an inconsistency or cycle in $\leq$.
By this observation, $E$ contains only $p_3$ or $p_4$.
Hence, $\{a_1,a_2\}$ is not conflict-free, contradiction.

\end{proof}



We now show that a game-based argument which is not a Nash equilibrium will not appear in any preferred extension of the associated argumentation system.

\begin{lemma}
Let $G = (\ag,\ac,\av, \ou, \ef, \pref)$ be a game, and $\AS_G$ be the corresponding AS. If $\Strat \in \Strat_G$ such that $\Strat$ is not a Nash equilibrium then for every preferred extension $E, \Strat \notin E$.
\label{lem:inpref-noNotNE}
\end{lemma}

\begin{proof}
Assume there is a non-Nash equilibrium game-based argument $S' = [s'_0, \dots, s'_n]$ in a preferred extension $E$. Then, from Lemma \ref{lem:pref-gameb}, $E$ does not contain any other game-based arguments.
    Since $S'$ is not a Nash equilibrium, there exists $i \in \ag$ and $s \in \ac_i$ such that $S'_{-i} \oplus s'_i <_i S'_{-i} \oplus s$.
    In the rest of this proof, we consider the strategy $s^*$ such that for every $s \in \ac_i, S'_{-i} \oplus s \leq_i S'_{-i} \oplus s^*$.
    By definition, the attack from $S'$ to $S'_{-i} \oplus s^*$ is attacked by the preference argument $(S'_{-i}, s^*)$. 
    Moreover, the preference argument $(S'_{-i}, s^*)$ attacks all the other preference arguments $(S'_{-i}, s')$, where $s' \in \ac_i$ and $s' \neq s$.
    By definition of the valuation arguments, none of the attacks from $(S'_{-i}, s^*)$ to those other preference arguments is defeated. 
    As a result, we conclude that there is preferred extension that contains $(S'_{-i}, s^*)$. 
    Let $s^+ = \{ s \in \ac_i \mid S'_{-i} \oplus s \leq_i S'_{-i} \oplus s^* $ and $S'_{-i} \oplus s^* \leq_i S'_{-i} \oplus s \}$, we can conclude that there is at least one argument $(S'_{-i}, s_o), s_o \in s^+$ in $E$ (Lemma \ref{lem:1}) and $(S'_{-i}, s_o)$ attacks the attack from $S'$ to $S'_{-i} \oplus s_o$, contradiction.
\end{proof}

\begin{corollary}
Let $G = (\ag,\ac,\av, \ou, \ef, \pref)$ be a game, and $\AS_G$ be the corresponding AS. If $E$ is a preferred extension that contains a game-based argument $\Strat$, then $\Strat$ is a Nash equilibrium.
\label{coro:GBAprefNE}
\end{corollary}


In the next proposition, we show that if there is only one preferred extension that contains a game-based argument, then there is an equivalence between preferred and stable extensions.

\begin{proposition}
Let $G$ be a game and $\AS_G = (\Arg, \Att, \Datt)$ be the corresponding argumentation framework. If $E \in \Ext_p(\AS_G) $ and $E \cap \A_g(G) \neq \emptyset$ then $ E \in \Ext_s(\AS_G) $.
\label{prop:preferredIsStable}
\end{proposition}

\begin{proof}
We show that if a preferred extension possesses a game-based argument, then it is also a stable extension.
Assume $E$ contains a single game-based argument. By Lemma \ref{lem:pref-gameb}, $E$ contains exactly one game-based argument. Therefore, all game-based arguments not in the extension are defeated by the game-based argument within the extension with respect to $E$, meaning that the game-based argument is a member (at the game-based level) of the stable extension.
\end{proof}

It may seem intuitive that the preferred and stable extension should coincide where multiple preferred extensions exist. However, this is not the case, as demonstrated by the following counter-example.

\begin{example}
\label{ex:noNash}
Consider the matching pennies game $G = (\ag, \ac,\av, \ou, \ef, \pref)$ where $\ag =\{0,1\}, \ac = \{heads,tails\}, \av= [\ac,\ac], \ou = \{1,-1\}$, $\pref$ is defined as the ``less-than relation" for each player, and $\ef$ is defined in Table \ref{tab:noNash}.


The set of arguments is $\Arg = \{ b_1, b_2, b_3, \ldots, b_{16} \}$ and are listed in Table \ref{tab:args-matching}.
%
%
There is only one preferred extension $\{b_{16}, b_{15}, b_{14}, b_{13}, b_{12}, b_{10}, b_{8}, b_{6} \}$ but no stable extensions.

\begin{table}
    \centering
\begin{tabular}{|c|c|c|}
\hline
Game-based arguments & Preference arguments& Valuation arguments\\
\hline
 $b_1 = [heads,heads]$&$b_5 = ([heads,\emptyset], heads)$ &$b_{13} = ([heads,\emptyset], tails > heads)$\\
     $b_2 = [heads,tails]$&$b_6 = ([heads,\emptyset], tails)$&$b_{14} = ([\emptyset,tails], tails > heads)$\\
     $b_3 = [tails,tails]$&$b_7 = ([\emptyset, tails], heads)$&$b_{15} = ([tails,\emptyset], heads > tails)$\\
    $b_4 = [tails,heads]$&$b_8 = ([\emptyset, tails], tails)$ &$b_{16} = ([\emptyset, heads], heads > tails)$\\
    &$b_9 = ([tails,\emptyset], tails)$ &\\
&$b_{10} = ([tails,\emptyset], heads)$& \\
&$b_{11} = ([\emptyset, heads], tails)$& \\  
&$b_{12} = ([\emptyset,heads], heads)$ &\\

\hline
\end{tabular}
    
    \caption{Arguments for the matching pennies game}
    \label{tab:args-matching}
\end{table}


\end{example}

Furthermore, even when multiple preferred extensions exist, these may not coincide with the stable extensions.

\begin{example}

\label{ex:noNash2pref}
Let us consider the following variant of the matching pennies game with three strategies for each player. We have $G = (\ag, \ac,\av, \ou, \ef, \pref)$ where $\ag =\{0,1\}, \ac = \{heads,tails,edge\}, \av= [\ac,\ac], \ou = \{1,-1\}$, $\pref$ is defined as the "less-than" relation for numbers for each player, and $\ef$ is defined in Table \ref{tab:noNash2pref}. This variant of the game has eight distinct preferred extensions, but none contain any game-based arguments.

\begin{table}
\centering
\begin{game}{3}{3}[Player~0][Player~1]
& $heads$ & $tails$ & $edge$ \\
$heads$ &$\stackedpayoffs{-1}{1}$ &$\stackedpayoffs{1}{-1}$&$\stackedpayoffs{1}{-1}$\\
$tails$ &$\stackedpayoffs{1}{-1}$ &$\stackedpayoffs{-1}{1}$&$\stackedpayoffs{-1}{1}$ \\
$edge$ &$\stackedpayoffs{1}{-1}$ &$\stackedpayoffs{-1}{1}$&$\stackedpayoffs{-1}{1}$
\end{game}
\caption{Three strategy variant of the matching pennies game.}
\label{tab:noNash2pref}
\end{table}

\end{example}


We now turn to our main result, namely the equivalence of the Nash equilibrium with the game-based arguments found in the preferred extensions.

\begin{proposition}[Equivalence]
Let $G = (\ag,\ac,\av, \ou, \ef, \pref)$ be a game, and $\AS_G$ be the argument framework for the game. A strategy profile $\Strat = [s_0, \dots, s_n] \in \Strat_G$ is a Nash equilibrium iff there exists $E \in \Ext_p(\AS_G)$ such that $\Strat \in E$.
\label{prop:equi}
\end{proposition}

\begin{proof}
We split this proof in two parts:
\begin{itemize}
    \item[$(\Rightarrow)$]
   We need to show that if $\Strat$ is a Nash equilibrium, then it is within a preferred extension of $\AS_G$. 
   Let us consider the set of arguments $E = \{\Strat \} \cup \A_v(G) \cup \{ (\Strat_{-i}, s_i) \mid i \in \ag \}$.
   We now show that $E$ is a preferred extension of $\AS_G$. It is clear that $E$ is conflict-free as for every $x,y \in E, (x,y) \notin \Att$. Every argument in $\A_v(G)$ is acceptable w.r.t.\ $E$ as valuation arguments are not attacked.
       Every argument $a = (\Strat_{-i}, s_i)$ is also acceptable w.r.t.\ $E$ because for every $s' \in \ac_i$ and $s' \neq s_i$, the attacks from $ a' = (\Strat_{-i}, s')$ to $a$, is either not a defeat w.r.t. $E$ (if there is a valuation argument that attacks $(a',a)$) or it is a defeat but $a'$ is defeated by $a$ w.r.t.\ $E$.
       The argument $\Strat$ is also acceptable w.r.t.\ $E$ because for every $\Strat' \in \Strat_G$ and $\Strat' \neq \Strat$, the attack from $\Strat'$ to $\Strat$ is not a defeat w.r.t.\ $E$ as the arguments $(\Strat_{-i}, s_i)$ are attacking those attacks.
       We conclude that the set $E$ is admissible.
        Following Lemma \ref{lem:pref-gameb} and \ref{lem:1}, we conclude that $E$ is maximal for set inclusion as it contains all the valuation arguments, one preference argument per cluster and exactly one game-based argument.

    \item[$(\Leftarrow)$]
    We need to show that if $S$ is within a preferred extension, then $S$ is a Nash equilibrium. This follows directly from the result from Corollary \ref{coro:GBAprefNE}.
    
\end{itemize}

\end{proof}



Returning to the stable extensions, the following result shows that there is a one-to-one correspondence between the sets of Nash equilibria and the set of classes of stable extensions\footnote{We say two stable extensions are equivalent iff they have the same game-based argument}, where each Nash equilibrium $S$ corresponds to the class of stable extensions containing argument $S$.

\begin{corollary}
Let $G = (\ag,\ac,\av, \ou, \ef, \pref)$ be a game, and $\AS_G$ be the corresponding EAF. 
There is a bijection between $Y = \{ \Strat \in \Strat_G \mid \Strat$ is a Nash equilibrium$\}$ and 
$\{ \{E \in \Ext_s(\AS_G) \mid \Strat' \in E    \} \mid  \Strat' \in Y\}$
\end{corollary}

\begin{proof}
Follows directly from Proposition \ref{prop:equi} and Proposition \ref{prop:preferredIsStable}.
\end{proof}

\noindent Finally, we consider how many arguments an argumentation system containing representing a normal form game will contain.

\begin{proposition}[Number of arguments]
Let $G = (\ag,\ac,\av, \ou, \ef, \pref)$ be a game s.t.\ $|\ag|= n$ and $m = \max\limits_{i \in \ag} |\ac_i|$, the number of arguments in $\AS_G$ is in $\mathcal{O} (m^{n+1} \cdot n) $. 
\end{proposition}

\begin{proof}
The proof is split into three parts.
\begin{enumerate}
    \item Suppose $n$ players and  $m$ strategies per player. Each game-based argument corresponds to a pure strategy profile, i.e., there are $m^n$ game-based arguments. 

\item Consider the number of the preference arguments. There are $m^{n-1} \cdot n$ partial strategy  profiles. Roughly speaking, a preference argument is obtained from a partial strategy profile by replacing the empty set with a strategy. Hence, there are up to $m^{n-1} \cdot n \cdot m = m^{n} \cdot n$ preference arguments. 

\item We estimate the number of valuation arguments. Each valuation argument is obtained from one partial strategy profile and one pair of different strategies. There are $m^{n-1} \cdot n$ partial strategy  profiles and up to $m \cdot (m-1)$ pairs of different strategies. Furthermore, if a strategy $x$ is preferred to strategy $y$, then $y$ is not preferred to $x$. Thus, there are up to $\frac{m \cdot (m-1)}{2}$ possible combinations to consider. Hence, the total number of valuation arguments is limited by 
$\frac{m^{n-1} \cdot m \cdot (m-1) \cdot n}{2}$
which is in $\mathcal{O}(m^{n+1} \cdot n)$. 
Thus, the total number of arguments is in $\mathcal{O}(m^{n}) + \mathcal{O}(m^{n} \cdot n) + \mathcal{O}(m^{n+1} \cdot n)$ which is in $\mathcal{O} (m^{n+1} \cdot n) $. 
\end{enumerate}
\end{proof}

We note that computing Nash equilibria is known to be computationally difficult, and the result regarding the number of arguments is therefore unsurprising.


\section{Discussion, Related and Future Work} \label{sec:discussion}

In this paper, we described how normal form games can be given an argumentation-based interpretation so as to allow -- via argumentation semantics -- for pure Nash equilibria to be computed. Intuitively, a Nash equilibrium identifies the best strategy a player can pursue given others' strategies. However, explaining -- to a non-expert -- why some set of strategies forms a Nash equilibrium is often difficult, and our argument-based interpretation is the first step towards an explanatory dialogue for such explanation. Other work has shown the utility of providing such dialogue-based explanations \cite{caminada2014ScrutablePlanEnactment,kristijonas2016ExplanationCaseBasedReasoning,oren2020ArgumentBasedPlanExplanation}. In the current context, such an explanation could build on Modgil's proof dialogues for extended argumentation frameworks \cite{modgilLabellingsGamesExtended}, and could result in a dialogue as follows for the Stag hunt game shown in Figure \ref{fig:GPd}.

\begin{description}
\item [User] ``Why should both players hunt a stag?" (why $a_2$?)
\item [System] ``It is the best response because $a_2$ defeats all the other game-based arguments, namely $a_1$, $a_3$ and $a_4$''.
\end{description}

\noindent Assume now that the user agrees that $a_2$ defeats $a_3$ and $a_4$; hence they ask further about why $a_2$ defeats $a_1$.

\begin{description}
\item [User:] ``Why should player $1$ play stag if player $0$ plays stag?" (why does $a_2$ defeats $a_1$?)
\item [System:] ``Because playing stag gives a better outcome to player $1$ if player $0$ plays stag" ($a_5$ defeats the attack $(a_1,a_2$))

\item [User:] ``Why does player $1$ not prefer the outcome when hare is played"? (why not $a_6$)?

\item [System:] ``Because of the valuation defined for player $1$" ($a_{13}$)

\item [User:] ``I understand."
\end{description}

In the short term, we intend to formalise the dialogue and empirically evaluate its explanatory capability with human subjects. Other extensions which we intend to investigate include providing an argumentation semantics for mixed Nash equilibria (perhaps through the use of some form of ranking semantics \cite{amgoud2016RankingArgumentsCompensationBased,bonzon2016ComparativeStudyRankingbased,matt2008GameTheoreticMeasureArgument}), and investigating other solution concepts (e.g., Pareto optimality) for more complex types of games. Finally, there are clear links between game theory and group-based practical reasoning. Building on work such as  \cite{atkinson16argument,shams20argumentation}, we intend to investigate how an argument-based formulation to practical reasoning underpinned by game theory can be created.

Several other authors have investigated some links between game theory and argumentation. For example, in his seminal paper, Dung \cite{dung1995AcceptabilityArgumentsIts} noted that the stable extension corresponds to the stable solution of an cooperative $n-$person game, but did not seem to deal with non-cooperative games as we do here. Game theory was also used to describe argument strength by Matt and Toni \cite{matt2008GameTheoreticMeasureArgument}, and Rahwan and Larson \cite{rahwan2009ArgumentationGameTheory} investigated the links between argumentation and game theory from a mechanism design point of view. Perhaps most closely related to the current work is Fan and Toni's work \cite{fanInterplayGamesArgumentation2016} exploring the links between dialogue and assumption-based argumentation (ABA). Here, the authors showed how admissible sets of arguments obtained from their ABA constructs are equivalent to Nash equilibria. In contrast to the current work, they only considered two player games and utilised structured argumentation, allowing them to describe a proof dialogue with associated strategies.

\section{Conclusions} \label{sec:conclusions}
In this paper, we provided an argumentation-based interpretation of pure strategies in normal form games, demonstrating how argumentation semantics can be aligned with the Nash equilibrium as a solution concept, and examining some of the argumentation system's properties. 

We believe that this work has significant application potential in the context of argument-based explanation. At the same time, we recognise that there are significant open avenues for research in this area, but believe that the current work is an important  step in investigating the linkages between the two domains.


\bibliographystyle{abbrv}
\bibliography{newbib}
\end{document}